\documentclass[12pt]{amsart}
\usepackage{latexsym, amsmath, amscd, amssymb, amsthm} 
\usepackage[T2A]{fontenc}
 \usepackage[cp1251]{inputenc}
\usepackage[english]{babel}

\newtheorem{te}{Theorem}
 
 \newtheorem{pr}{Example}[section]

\def\mathbi#1{\textbf{\em #1}}
\begin{document}

\noindent

 \title[$2D$  moment invariants]{ 2D  moment invariants from the point of view of the classical  invariant theory }

\author{Leonid Bedratyuk}\address{Khmelnitskiy national university, Insituts'ka, 11,  Khmelnitskiy, 29016, Ukraine}

\begin{abstract} 

  Invariants allow to classify images  up to the action of a group of transformations. In this paper we introduce   notions of the algebras of simultaneous polynomial and rational 2D moment invariants and 
 prove that   they are isomorphic  to the algebras of joint  polynomial and rational $SO(2)$-invariants of  binary forms. 
Also, to simplify the calculating of invariants we pass from an action of Lie group $SO(2)$ to  an action of its Lie algebra $\mathfrak{so}_2$.  
This allow us to reduce the problem to standard  problems of the classical invariant theory.

\end{abstract}

\maketitle

\section{Introduction}

Let  $\mathbi{F}$ be a set of   real finite piecewise continuous functions  that  can have nonzero values only in 
 a compact subset of  $\mathbb{R}^2.$  Let a transformation group $G$ acts on  $\mathbb{R}^2$. This action induces an   action on $\mathbi{F}$  by the rule
$$g \cdot f(u)=f(g^{-1}u), g \in G, f \in \mathbi{F}, u \in \mathbb{R}^2.
$$
 The orbit of a function  $f \in \mathbi{F}$  is defined as  the set  $ \{ g \cdot f \mid g \in G\}$. Then   \mathbi{F} is a disjoint union of orbits.    

A problem of great practical interest in applications is the following classification problem:
given two functions $f_1, f_2 \in F$, are they   in  the same $G$-orbit?  
A common approach to this problem, see  \cite{Ri}-\cite{KM},  is to use invariant functionals.
A functional  $I:\mathbi{F} \to \mathbb{C}$  is called  \textit{invariant  } with respect to the group  $G$  ( or $G$-invariant) if 
$$
I(g \cdot f)=I(f), \forall g \in G, \forall f \in \mathbi{F}.
$$
 The $G$-invariant functionals being  closed under addition, multiplication, scalar multipliplication  and division  form two algebraic objects --  the algebra of polinomial invariants    $\mathbb{C}[F]^G$ and the field of rational invariants     $\mathbb{C}(F)^G.$
 
It is easy to see that the invariant functional remains constant on the orbits. The set of invariant functionals   $I_1, I_2, \ldots$ is called  \textit{complete } if they  separate the orbits, i.e. for all   $f_1, f_2 \in \mathbi{F}$ the equalities 
$$
I_k(f_1)=I_k(f_2), k=1,2,\ldots
$$
imply that  $f_1, f_2 $ lie in the same orbit and  then we have that  $f_1=g f_2$ for some $g \in G.$
So,  this classification problem is reduced to problems  of choosing   suitable functionals and then to deriving  of a complete set of invariants.

The  orientation-preserving transformation group  $G$  is widely used in    2D image analysis. The group     is  the semi-direct product of the plane translation group  $TR(2)$, the   direct product of the  plane rotation group $SO(2)$  and the  uniform scaling  group $\mathbb{R}^*$:   
$$
G=(R^* \times SO(2)) \rtimes TR(2) 
$$
Also,    $G$ is a four-parametric group with the following action on   $(x,y) \in \mathbb{R}^2:$ 
$$
(x,y)^{\top} \mapsto  s \begin{pmatrix}  \cos\theta & -\sin\theta\\ \sin\theta &\cos\theta \end{pmatrix} \begin{pmatrix}  x\\ y \end{pmatrix}+\begin{pmatrix}  a\\ b \end{pmatrix}.
$$

The introduction  of the notion of  image moment invariants by Hu in his significant paper \cite{Hu}  was a   
vivid  example of the application of the classical invariant theory to the pattern recognition.
  	He proposed to consider as functionals  a well-known statistical moments.
		More specifically,  Hu have  considered  \textit{ the geometric moments } of  $f \in \mathbi{F}:$

 \begin{gather*}
m_{pq}(f(x,y))=m_{pq}=\iint\limits_{\Omega} x^p y^q f(x,y) dx dy, \Omega \subset \mathbb{R}^2.
\end{gather*}

and the \textit{central geometric moment}

$$
\mu_{pq}(f(x,y))=\mu_{pq}=\iint\limits_{\Omega} (x-\bar{x})^p (y-\bar{y})^q f(x,y) dx dy,
$$
where
$$
\bar{x}=\frac{m_{10}}{m_{0,0}},\bar{y}=\frac{m_{01}}{m_{0,0}}.
$$
The central geometric moments are already invariants under the translation group.  After the  normalization 
$$
\eta_{p,q}=\frac{\mu_{p,q}}{\mu_{0,0}^{1+\frac{p+q}{2}}}, p+q \geq 2,
$$
they become invariants  of the scaling group. Therefore, the problem of determining $G$-invariants  can be  reduced to  problem of finding $SO(2)$-invariants as functions of the normalized central geometric moments. We will  consider two types of such functions -- polynomials and rational functions.

Let  $\mathbb{C}[\eta]$ be a polynomial algebra  in  countable many  variables $\{\eta_{p,q}\}$ considered with the action of the group  $SO(2).$
Denote  by   $\mathbb{C}[\eta]^{SO(2)}$ and  $\mathbb{C}(\eta)^{SO(2)}$ the corresponding algebras of \textit{polynomial}  and  \textit{rational  moment }invariants, respectively.  Since these algebras are not finitely generated, then  a complete set of invariants consists of infinitely many invariants. However, these algebras can be approximated by the finitely generated algebras $\mathbb{C}[\eta]_d^{SO(2)}$ and  $\mathbb{C}(\eta)_d^{SO(2)}$where $[\eta]_d=\{ \eta_{p,q},2 \leq p+d \leq d\} $. The elements of these algebras are called the \textit{simultaneous}  2D moment  invariants   of \textit{order} up to $d.$ For small $d$ its  generators   already known. For example, the algebras $
\mathbb{C}[\eta]_{2}^{SO(2)} 
$ and $\mathbb{C}(\eta)_{2}^{SO(2)}$
generated by the two invariants
$$ \eta_{2,0}+\eta_{0,2},(\eta_{2,0}-\eta_{0,2})^2+4\eta_{1,1}^2.$$
For $d=3$  the algebra $\mathbb{C}(\eta)_{3}^{SO(2)}$ is generated by 11 simultaneous  moment  invariants. In the paper we prove that  algebra $\mathbb{C}[\eta]_{3}^{SO(2)}$  is generated by 14 simultaneous  moment  invariants.

It is surprising that in the general case the problem can be reduced to well-known problems of the classical invariant theory.  It turns out,  the algebra  $\mathbb{C}[\eta]_d^{SO(2)}$ is isomorphic  ( see {Theorem~1}) to the algebra  of \textit{joint $SO(2)$-invariants} of  some  system of binary forms.  
There is a more general result, namely  the \textit{fundamental  theorem  of  moment  invariants}, see \cite{Hu} and \cite{RR}, \cite{MM}, which  can be reformulated as follows: for any linear algebraic group $G \subseteq GL(2)$ considered with the action on $\mathbb{C}(\eta)$  the simultaneous moment invariants  algebra   $\mathbb{C}(\eta)_d^{G}$ is isomorphic to the  algebra  $\mathbb{C}(W_d)^{G}$ of the \textit{joint invariants} several binary forms. Here  $W_d$ is the following direct sum  
$$
W_d=V_2 \oplus V_3 \oplus \cdots \oplus V_d, 
$$
where  $V_k$  is the vector $G$-space  of  binary forms of order  $k$.

   An explicit description   of the  algebras of invariants of  binary forms   was an  important  object of research  in  invariant theory in  the 19th century. 
 From the point of view of  the classical invariant theory the problem of  description  of the algebras $\mathbb{C}[\eta]_d^{SO(2)}$,  $\mathbb{C}(\eta)_d^{SO(2)}$ can thus be formulated as follows.
Let the group $SO(2)$  acts by linear changes of variables on the vector space  $V_d$.  
Since   $SO(2)$ is a \textit{compact group}, 
the  $\mathbb{C}[W_d]^{SO(2)}$ is the finite generated algebra.  Moreover, since  $SO(2)$ is a \textit{reductive group}, the invariants of  $\mathbb{C}[\eta]_d^{SO(2)}$ (and $\mathbb{C}(\eta)_d^{SO(2)}$)   separate orbits, thus a complete  set of invariant always exists and is finite.    Otherwise speaking, if the images  $f_1, f_2 \in \mathbi{F}$  do not lie in the same  $G$-orbit then for some $d $ there exists an invariant  $\textit{I} \in \mathbb{C}[\eta]_d^{SO(2)}$ ( or $\mathbb{C}(\eta)_d^{SO(2)}$)   such that  $\textit{I}(f_1) \neq \textit{I}(f_2).$

The questions then arise:
\begin{itemize}
	\item 
\textbf{Problem~1.}  What is  a minimal generating set of  the algebra polynomial invariants  $\mathbb{C}[W_d]^{SO_2}?$

\item 
\textbf{Problem~2.} What is  a minimal generating set of  the algebra rational  invariants  $\mathbb{C}(W_d)^{SO(2)}$?
\end{itemize}

Problem~1 was solved partially by Hu for the case $d=3.$
Note that Hu used to  the two rotation invariants of quartic and  the four  invariants of cubic which were already derived in \cite[p.378]{Ell}. 

Many authors  tried to solve Problem~2  with varying degrees of accuracy and mathematical rigor, see \cite{Hu}-\cite{FSB}. For the first time  it was solved  by Flusser \cite{FF} in  a satisfactory way. The author presented a minimal generating set of  $\mathbb{C}(W_d)^{SO(2)}$ in the complex moment terms for arbitrary $d.$

In the present  paper, the problem  of finding  invariants of the Lie group  $SO(2)$  is reformulated  in terms  of its Lie algebra  $\mathfrak{so}_2.$  Since, the Lie algebra  $\mathfrak{so}_2$  is one-dimensional and it acts on  basis vectors $a_{k,j}$ of the dual vector space $W_d^*$ by the derivation  $D:$
$$
D=\sum_{k+j=2}^d \left( j a_{k+1,j-1}- k a_{k-1,j+1} \right) \frac{\partial}{a_{k,j}},
$$
then the problem of explicit description of  $\mathbb{C}(W_d)^{\mathfrak{so}_2}$ can be reduced   the problem of solving    a  partial differential  equation. 

The paper is arranged as follows. 

In  Section~2, we  review  basic concepts of the classical invariant theory. Also, we prove that  the algebras  $\mathbb{C}(\eta)_d^{SO(2)}$ and  $\mathbb{C}(W_d)^{\mathfrak{so}_2 }$ are  isomorphic.

In Section~3,   we prove that the derivation  $D$ is a diagonalizable linear  operator   on the dual vector space  $V_d^*$ with the spectrum  $ \{ -d i,-(d-2)i,-(d-4)i,\ldots, (d-2)i,d i \}$ and that for each  eigenvalue  $\pm s i$ the associated eigenvector  $\mathbi{e}_{d}({\pm s i})$ has the form $$\text{  {\rm  $\mathbi{e}_{d}({\pm s i})$}}=\sum_{j=0}^d i^j \mathcal{K}_j\left(\frac{1}{2}(d \pm s),d \right)a_{d-j,j},$$

where $$
\mathcal{K}_n(x,a)=\sum_{j=0}^n (-1)^j {x \choose j} {a-x \choose n-j},
$$
is   the binary Kravchuk polynomial  and $i$ is the  imaginary unit. In particular, we show that the complex moments also are   eigenvectors of  $D.$

In  section~4, by using the notion of Hilbert basis for the case $d \leq 5$  we present a minimal generating system for the algebra of rotation polynomial invariants   $\mathbb{C}[W_d]^{\mathfrak{so}_2}$.  
Namely, we show    that for $d=3$ its  minimal generating set  consists of   14  invariants, for $d=4$ it consists of  65 invariants and for $d=5$ it consists of  562 invariants. 

In Section~5,  for $\mathbb{C}[W_d]^{\mathfrak{so}_2}$ we prove the analogue of the  classical  Cayley-Sylvester  formula, namely   
the number of linearly independent homogeneous polynomial rotation  invariants  of degree  $n$ is equal to the number of   non-negative integer solutions of the system of equations: 
$$
\begin{cases} \displaystyle 
\sum_{k=1}^{2d}k(\alpha_{1}^{(k)}+\alpha_{2}^{(k)}+\cdots+\alpha_{l_k}^{(k)})=dn,\\
\displaystyle  \sum_{k=0}^{2d}(\alpha_{1}^{(k)}+\alpha_{2}^{(k)}+\cdots+\alpha_{l_k}^{(k)})=n,
\end{cases}
$$
where $\alpha_{j}^{(k)}$ are the unknowns. 
Also we  derive a formula for corresponding Poincar\'e series.

In  Section~6,  we prove that the number of elements in a minimal generating set of the algebra rational invariants $\mathbb{C}(W_d)^{\mathfrak{so}_2}$ equals $\dim W_d-1$. Also, we present such a minimal generating set, namely  the set 
$$
G^{(d)}_{p,q}=\left\{\mathbi{e}_{2j}(0), \mathbi{e}_{n}(si)\mathbi{e}_{n}(-si), \mathbi{e}_n(si)^q{e_p(-q i)^{s}} \mid 2 \leq n \leq d, 2j \leq d, s i \in \Lambda_n,  s >0, q \neq s \right\},
$$
 is a minimal generating set of the algebra  $\mathbb{C}(W_d)^{\mathfrak{so}_2}$ for   suitable parameters $p,q.$  Also, we express the invariants $\mathbi{e}_{2j}(0)$ and $ \mathbi{e}_{n}(si)\mathbi{e}_{n}(-si)$ in terms of the Kravchuk polynomials.



\section{Basic notions  of  the invariant theory }


Classical invariant theory is centered around the action of the general linear group on homogeneous
polynomials, with an emphasis on binary forms.

Let us recall that the space of \textit{binary forms} of degree $d$ is the vector space:
$$
V_d=\left\{ \sum_{k+j=d} \binom{d}{k} a_{k,j} x^k y^j \mid  a_{k,j} \in \mathbb{C} \right\}.
$$

The linear functions  
$$
 \sum_{k+j=d} \binom{d}{k} a_{k,j} x^k y^j \mapsto a_{k,j}
$$
form a basis of the dual vector space  $V_d^*$. For convenience, it is useful to  identify the function and  the coefficient $a_{k,j}$.
Moreover, these functions generate an algebra. The algebra  is called the \textit{algebra of polynomial functions} on    $V_d$ and  denoted by  $\mathbb{C}[V_d].$ Thus,  any element of $\mathbb{C}[V_d]$ is just a polynomial of  variables $\{a_{d,0},a_{d-1,1},\ldots,a_{0,d}\}.$   The algebra   $\mathbb{C}[V_d]$  is \textit{$\mathbb{N}$-graded algebra}:
$$
\mathbb{C}[V_d]=(\mathbb{C}[V_d])_0 \oplus (\mathbb{C}[V_d])_1  \oplus \cdots, (\mathbb{C}[V_d])_s \cdot (\mathbb{C}[V_d])_t \subseteq (\mathbb{C}[V_d])_{s+t},
$$
where each  $(\mathbb{C}[V_d])_k$  is the  vector space of all homogeneous polynomials of degree $k.$

Let  $GL(V_d)$ be the group of all invertible linear transformations of 
$V_d.$  If  $g \in GL(V_d)$, $F \in \mathbb{C}[V_d]$ define a new polynomial function  $ g \cdot F \in \mathbb{C}[V_d]$ by 
$$
(g \cdot F)(v)=F\left(  g^{-1} v\right).
$$
If $G$ is subgroup of  $GL(V_d)$ we say that   $F$ is  \textit{$G$-invariant} if $ g \cdot F=F$ for all  $g \in G.$
The $G$-invariant polynomial functions form a graded 
subalgebra $\mathbb{C}[V_d]^G$ of $\mathbb{C}[V_d]$. The algebra $\mathbb{C}[V_d]^G$  is called the \textit{algebra of the polynomial $G$-invariants}  of  the vector space $V_d$.

Let us recall that \textit{a derivation} of a ring   $R$ is an additive  map  $L$ satisfying the Leibniz rule: 
$$
L(r_1 \, r_2)=L(r_1) r_2+r_1 L(r_2), \text{  for all }  r_1, r_2 \in R.
$$
The subring 
$$
\ker L:=\left \{ f \in R|  L(f)=0 \right \},
$$
is called \textit{the kernel} of the derivation $L.$

Let  $G$ be  a simply connected Lie group acting on $V_d$  ( and on $\mathbb{C}[V_d]$)   and  let $\mathfrak{g}$ be its  Lie algebra. It is well-known that the algebra  
$\mathfrak{g}$  acts on  $\mathbb{C}[V_d]$ by derivations and  $ F \in \mathbb{C}[V_d]^G$  implies that  $L(F)=0$, $\forall L \in \mathfrak{g}.$ Thus 
$$
\mathbb{C}[V_d]^G=\mathbb{C}[V_d]^\mathfrak{g}.
$$ 
As a linear object, a Lie algebra is often a lot easier to work with than working directly with the corresponding Lie group.

Let us consider two important examples.

\begin{pr}{\rm Let  $SL(2)$ be the group of   $2 \times 2$-matrix  with determinant  1 equipped with the natural action on   
$V_d$. The matrices with zero trace  
$$
\begin{pmatrix} 0 & 0 \\ 1 & 0  \end{pmatrix}, \begin{pmatrix} 0 & 1 \\ 0 & 0  \end{pmatrix},
$$
generate the Lie algebra  $\mathfrak{sl}_2$ and act on  $V_d$ by the derivations
$$
-y \frac{\partial }{\partial x}, -x \frac{\partial }{\partial y}.
$$
and act on  $\mathbb{C}(V_d)$ by the derivations
\begin{gather*}
D_{+}=\sum_{i+j=d} ja_{k+1,j-1}\frac{\partial}{a_{k,j}},\\
D_{-}=\sum_{i+j=d} ka_{k-1,j+1}\frac{\partial}{a_{k,j}}.
\end{gather*}
The polynomial solutions of the corresponding system of differential equations generate the algebra 
 $\mathbb{C}[V_d]^{\mathfrak{sl}_2}$.  For instance, for  $d=2$ the algebra of invariants generated by the discriminant of binary form   $a_{2,0}x^2+2 a_{1,1}xy+a_{0,2}y^2$:
$$
\mathbb{C}[V_2]^{\mathfrak{sl}_2}=\ker D_+ \cap \ker D_-=\mathbb{C}[a_{1,1}^2-a_{2,0}a_{0,2}].
$$
It is easy to check that 
$$
D_{+}(a_{1,1}^2-a_{2,0}a_{0,2})=D_{-}(a_{1,1}^2-a_{2,0}a_{0,2})=0.
$$
The  minimal generating systems  of $\mathbb{C}[V_d]^{\mathfrak{sl}_2}$ were  a major object of research in classical invariant theory of the 19th century. At present, 
such generators have been found only for $d \leq 10.$ 

}
\end{pr}

\begin{pr}\label{so2}{\rm  Let $$SO(2)=\left\{ \begin{pmatrix}  \cos\theta & -\sin\theta\\ \sin\theta &\cos\theta \end{pmatrix} \mid  \theta \in \mathbb{R} \right\},$$ be the special orthogonal group. This group is also called the \textit{rotation group}.   $SO(2)$ is connected, compact subgroup of the group  $SL(2)$. The corresponding Lie algebra   $\mathfrak{so}_2$  is generated by the matrix
$$ 
\frac{d }{d \theta} \begin{pmatrix}  \cos\theta & -\sin\theta\\ \sin\theta &\cos\theta \end{pmatrix}^{-1} \Bigg|_{\theta=0} = \left( \begin{array}{cc}
0  & 1 \\
-1 &  0 \end{array} \right),
$$
and acts on  $\mathbb{C}(V_d)$ by derivation
$$
D=\sum_{k+j=d} \left( ja_{k+1,j-1}\frac{\partial}{a_{k,j}}- ka_{k-1,j+1}\frac{\partial}{a_{k,j}} \right)=D_{+}-D_{-}
$$
see for example \cite[page 350]{Ell}  and \cite{FS}.
We have 
\begin{align*}
&\mathbb{C}[V_1]^{\mathfrak{so}_2}=\ker D=\mathbb{C}[a_{1,0}^2+a_{0,1}^2],\\
&\mathbb{C}[V_2]^{\mathfrak{so}_2}=\ker D=\mathbb{C}[a_{2,0}+a_{0,2},(a_{2,0}-a_{0,2})^2+4a_{1,1}^2].
\end{align*}
 }
\end{pr}

Note that  $\mathbb{C}[V_d]^{\mathfrak{sl}_2} \subset \mathbb{C}[V_d]^{\mathfrak{so}_2},$ thus any $\mathfrak{sl}_2$-invariant of binary form also is a rotation invariant.

In the similar manner we define an action of   $\mathfrak{so}_2$  on the direct sum 
$$
W_d=V_2 \oplus V_3 \oplus V_4 \oplus \cdots \oplus V_d.
$$
The corresponding algebras of polynomial and rational invariants are denoted by  $\mathbb{C}[W_d]^{\mathfrak{so}_2}$  and  $\mathbb{C}(W_d)^{\mathfrak{so}_2}$. The number $d$ is called the \textit{order} of an invariant.

 Let us now  return to the moment invariants. 
Observe, that the invariants of  $\mathbb{C}[V_2]^{\mathfrak{so}_2}$  have the same form as the first two Hu's moment invariants. For the  general case the following statement holds.

\begin{te} The algebras of polynomial and rational moment invariants $\mathbb{C}[\eta]_d^{SO(2)}$ and  $\mathbb{C}(\eta)_d^{SO(2)}$ are isomorphic to the algebras  $\mathbb{C}[W_d]^{\mathfrak{so}_2}$  and   $\mathbb{C}(W_d)^{\mathfrak{so}_2}$ respectively.
\end{te}
\begin{proof}
The group  $SO(2)$ acts on a normalized moment $\eta_{k,j}$, see \cite{Te}, \cite{PR1}, in a such way
$$
\begin{pmatrix}  \cos\theta & -\sin\theta\\ \sin\theta &\cos\theta \end{pmatrix}^{-1} \eta_{k,j}=\sum_{k=0}^p \sum_{j=0}^q (-1)^{p-k}\binom{p}{k}\binom{q}{j} (\cos\theta)^{q+k-j} (\sin\theta)^{p-k+j}\eta_{k+j, p+q-k-j}.
$$
To get action of the Lie algebra  $\mathfrak{so}_2$ we differentiate it by $\theta$ and after simplification obtain:
\begin{gather*}
\frac{d}{d \theta} \begin{pmatrix}  \cos\theta & -\sin\theta\\ \sin\theta &\cos\theta \end{pmatrix}^{-1} \eta_{p,q} \Big|_{\theta=0}=q \eta_{p+1,q-1}-p \eta_{p-1,q+1}.
\end{gather*}

Thus, the algebra  $\mathfrak{so}_2$ acts on $W_d$ as the differential operator
$$
\sum_{p+q=2}^d( q \eta_{p+1,q-1}-p \eta_{p-1,q+1})\frac{\partial}{\partial \eta_{p,q}}.
$$
But the action is identical to one described in Example~2.2.

Thus, the normalized moment invariant and the  $\mathfrak{so}_2$-invariants of the sum of binary forms are defined by the same partial differential equation. It implies that  $\mathbb{C}(\eta)_d^{SO(2)} \cong \mathbb{C}(W_d)^{\mathfrak{so}_2}$  and $\mathbb{C}[\eta]_d^{SO(2)} \cong \mathbb{C}[W_d]^{\mathfrak{so}_2}$
\end{proof}

Let us recall, that 
 a \textit{finitely generated algebra}  is an associative algebra $A$ over a field $\mathbb{K}$  where there exists a finite set of elements $S$ of $A$ such that every element of $A$ can be expressed as  a polynomial in $S$ with coefficients in $\mathbb{K}$. The set $S$ is called a \textit{generating set} of the algebra  
 $A.$   A generating set of $A$   is \textit{minimal} if no proper subset of it generates $A$.
If $A$ is finitely generated graded algebra then any two its minimal generating sets 
 have the same number of homogeneous elements. 
If the algebra $A$ is a field extension  over $\mathbb{K}$ then its generators are algebraically  independed and the size of a minimal generating set is equal to the transcendence degree of the field extension  $A/\mathbb{K}$. 


\section{Eigenvalues, eigenvectors of the derivation   ${D}$  and the Kravchuk polynomials. }


Since the algebra   $\mathfrak{so}_2$ is commutative  then all its irreducible representations are one-dimensio\-nal. It implies that the derivation operator $D$ is diagonalizable on $V_d^*$ and to find the algebra   $\mathbb{C}[V_d]^{\mathfrak{so}_2}$ it is enough to find  its  eigenvalues and  eigenvectors.

As it be shown above, see Example~\ref{so2}, the action of $D$ on the basis elements  $ a_{d,0},\ldots, a_{0,d}$ of  $V_d^*$ is as follows:
$$
D(a_{p,q})=q a_{p+1,q-1}-p a_{p-1,q+1}.
$$
 
It is easy to see that     the matrix  $M_d$  of $D$ is the basis has the form 
$$
M_d=\left( \begin {array}{cccccc} 0&1&\ldots&0&0&0\\ \noalign{\medskip}-d&0&2&\ldots&0&0\\ 
\noalign{\medskip} \ldots&\ddots&\ldots&\ddots&\ldots&\ldots\\
\noalign{\medskip}0&-(d-1)&\ldots&\ldots&0&0\\ \noalign{\medskip}0&0&\ddots&0&d-1
&0\\ \noalign{\medskip}0&0&\ldots&-2&0&d\\ \noalign{\medskip}0&0&\ldots&0&-1&0
\end {array} \right).
$$
The following statement holds
\begin{te}\label{1} Let  $\Lambda_d$  be the spectrum of the matrix  $M_d$ and let  $\text{  {\rm  $\mathbi{e}_{d}({\lambda})$}}$ be  the  eigenvector associated with the eigenvalue  $\lambda  \in  \Lambda_d$. Then 
$$
\begin{array}{ll} 
(i) & \displaystyle \text{$M_d$ is diagonalizable  and }  \Lambda_d=\{ -d i,-(d-2)i,-(d-4)i,\ldots, (d-2)i,d i \},\\ 
(ii) &   \displaystyle \text{  {\rm  $\mathbi{e}_{d}({\lambda})$}}=\sum_{j=0}^d i^j \mathcal{K}_j\left(\frac{1}{2}(d-Im(\lambda)),d \right)a_{d-j,j},
\end{array}
$$
here $i^2=-1$, $Im(z)$ is the imaginary part of a complex number  $z$ and 
$$
\mathcal{K}_n(x,a)=\sum_{j=0}^n (-1)^j {x \choose j} {a-x \choose n-j},
$$
is  the binary Kravchuk polynomial.
\end{te}
\begin{proof}
$(i)$  Consider the  $(d+1) \times (d+1)$-matrix
$$
{\rm S}_d=\left( \begin {array}{cccccc} 0&1&\ldots&0&0&0\\ \noalign{\medskip}d&0&2&\ldots&0&0\\
\noalign{\medskip} \ldots&\ldots&\ldots&\ldots&\ldots&\ldots\\
 \noalign{\medskip}0&d-1&\ddots&\ddots&0&0\\ \noalign{\medskip}0&0&\ldots&0&d-1
&0\\ \noalign{\medskip}0&0&\ldots&2&0&d\\ \noalign{\medskip}0&0&\ldots&0&1&0
\end {array} \right).
$$

The matrix  ${\rm S}_d$  is    the tridiagonal  Sylvester matrix. It is  diagonalized   and has $d+1$ distinct  eigenvalues: 
$$
\sigma ({\rm S}_d)=\{ -d,-(d-2),-(d-4),\ldots, (d-2), d \},
$$
see  \cite{JS} and \cite{EK}. 
Let us consider now the diagonal matrix  ${\rm diag}(1,i,i^2,\ldots,i^d)$. By  direct calculation we get  
$$
{{\rm diag}(1,i,i^2,\ldots,i^d)}\,  M_d \, {{\rm diag}(1,i,i^2,\ldots,i^d)^{-1}}=-i {\rm S}_d.
$$
Thus
$$
\Lambda_d=-i\sigma ({\rm S}_d)=\{ -d i,-(d-2)i,-(d-4)i,\ldots, (d-2)i,d i \}.
$$
$(ii)$
Let  $\mathbi{e}_{n}({\lambda})=(x_0,x_1,\ldots,x_d)^{\top}$ be an eigenvector of  $M_d$  associated with the eigenvalue  $\lambda \in \Lambda_d$. Let us find a recurrence relation for the sequence   $\{x_k \}$. By the definition of an eigenvector the following identity holds:  $$ M_d \mathbi{e}_{n}({\lambda}) = \lambda \mathbi{e}_{n}({\lambda}).$$ 
We have
\begin{gather*}
\left(\begin {array}{cccccc} 0&1&\ldots&0&0&0\\ \noalign{\medskip}-d&0&2&\ldots&0&0\\ 
\noalign{\medskip} \ldots&\ddots&\ldots&\ddots&\ldots&\ldots\\
\noalign{\medskip}0&-(d-1)&\ldots&\ldots&0&0\\ \noalign{\medskip}0&0&\ddots&0&d-1
&0\\ \noalign{\medskip}0&0&\ldots&-2&0&d\\ \noalign{\medskip}0&0&\ldots&0&-1&0
\end {array} \right) \cdot \begin{pmatrix} x_0\\x_1\\x_2\\ \ldots \\  x_k \\ \ldots\\ x_{d-1}\\ x_d\end{pmatrix}  =
\begin{pmatrix} x_1 \\-dx_0+2x_2\\ -(d-1)x_1+3x_3\\ \ldots\\-(d+1-k)x_{k-1}+(k+1)x_{k+1}\\ \ldots \\ -3x_3-z^{d-1} \\-x_{d-1}\end{pmatrix}.
\end{gather*}
After multiplication we get the following system of linear equations 
$$
\begin{cases} \lambda x_0=x_1,\\
\lambda x_1=-dx_0+2x_2,\\
\ldots \\
\lambda x_{k} =-(d+1-k)x_{k-1}+(k+1)x_{k+1},\\
\ldots \\
\lambda x_d=-x_{d-1}.
\end{cases}
$$
Multiplying the members of the equations by 
 $z^k, k=0,1,\ldots,d$ and adding, we obtain 
\begin{gather*}
\sum_{k=0}^d \lambda x_k z^k=\lambda \sum_{k=0}^d  x_k z^k=\lambda g(z),
\end{gather*}
where
$$
g(z)=x_0 +x_1 z+\cdots+x_d z^d.
$$
After simplification, we get 
\begin{gather*}
x_1+(-dx_0+2x_2)z+(-(d-1)x_1+3x_3)z^3+\cdots+(-3x_3-z^{d-1})-x_d z^d=\\=(1+z^2)g'(z)-d \cdot z  g(z).
\end{gather*}
Finally, we obtain the linear differential equation 
 $$
(1+z^2)g'(z)-dz g(z)=\lambda g(z),
$$
or  
$$
g'(z)=\frac{\lambda+dz}{1+z^2}g(z).
$$
We solve the differential equation subject to the initial condition  $g(0)=1$ and get the explicit form of  the generating function for the sequence $\{ x_j \}$:
$$
g(z)=(1+z^2)^{\frac{d}{2}}e^{\lambda \arctan(z)}.
$$
Since
\begin{gather*}
e^{i  \arctan(z)}=\cos(\arctan(z))+i \sin( \arctan(z))= \frac{1+iz}{\sqrt{1+z^2}},
\end{gather*}
then 
\begin{gather*}
g(z)=(1+z^2)^{\frac{d}{2}}e^{\lambda \arctan(z)}=(1+z^2)^{\frac{d}{2}}e^{i Im(\lambda) \arctan(z)}=(1+z^2)^{\frac{d}{2}} \left(\frac{1+iz}{\sqrt{1+z^2}}\right)^{Im(\lambda)}=\\=(1+z^2)^{\frac{d-Im(\lambda)}{2}}(1+iz)^{Im(\lambda)}=(1+iz)^{\frac{d+Im(\lambda)}{2}}(1-iz)^{\frac{d-Im(\lambda)}{2}}.
\end{gather*} 
Let us recall that the binary Kravchuk polynomials  $\mathcal{K}_n(x,a)$  are defined by 
$$
\mathcal{K}_n(x,a):=\sum_{j=0}^n (-1)^j {x \choose j} {a-x \choose n-j},
$$
with, see \cite{Is}, the ordinary generating function: 
$$
\sum_{j=0}^{\infty} \mathcal{K}_j(x,a) z^k=\left( 1+z \right) ^{a-x} \left( 1-z \right) ^{x}.
$$
Therefore
\begin{gather*}
g(z)=(1+iz)^{\frac{d+Im(\lambda)}{2}}(1-iz)^{\frac{d-Im(\lambda)}{2}}=(1+iz)^{d-\frac{d-Im(\lambda)}{2}}(1-iz)^{\frac{d-Im(\lambda)}{2}}=\\=
\sum_{j=0}^{\infty} \mathcal{K}_j\left(\frac{d-Im(\lambda)}{2},d \right) (iz)^j.
\end{gather*}

Thus, we have the following expression for the eigenvectors of the derivation $D:$ 
$$
\mathbi{e}_{d}({\lambda})=\sum_{j=0}^d i^j \mathcal{K}_j\left({ \frac{1}{2}(d-Im(\lambda))},d \right)a_{d-j,j},
$$

or in  the notation   $\lambda=\pm i s$:
\begin{align*}
&\mathbi{e}_{d}({si})=\sum_{j=0}^d i^j \mathcal{K}_j\left(\frac{1}{2}(d-s),d \right)a_{d-j,j},\\
&\mathbi{e}_{d}({-si})=\sum_{j=0}^d  i^j \mathcal{K}_j\left(\frac{1}{2}(d+s),d \right) a_{d-j,j}.
\end{align*}
\end{proof}
Note  that $\overline{\mathbi{e}_{n}({\lambda})}=\mathbi{e}_{n}({-\lambda})$, here  $\overline{\phantom{a^b}}$ denotes the complex conjugate.

\begin{pr}{\rm For  $W^*_3=V_2^* \oplus V_3^*$ we have   7 eigenvectors --  3 eigenvectors in  the vector space   $V_2^*=\langle a_{2,0},a_{1,1},a_{0,2} \rangle:$
\begin{align*} 
&\mathbi{e}_2(2i)=a_{{2,0}}+2\,ia_{{1,1}}-a_{{0,2}},\\
&\mathbi{e}_2(0)=a_{{2,0}}+a_{{0,2}},\\
&\mathbi{e}_2(-2i)=e_2(2i)^*=a_{{2,0}}-2\,ia_{{1,1}}-a_{{0,2}},
\end{align*}
and  4 eigenvectors in  the vector space  $V_3^*=\langle a_{3,0},a_{2,1},a_{1,2},a_{0,3} \rangle:$
\begin{align*}
&\mathbi{e}_3(3i)=a_{{3,0}}+3\,ia_{{2,1}}-3\,a_{{1,2}}-ia_{{0,3}},\\
&\mathbi{e}_3(i)=a_{{3,0}}+ia_{{2,1}}+a_{{1,2}}+ia_{{0,3}},\\
&\mathbi{e}_3(-i)=a_{{3,0}}-ia_{{2,1}}+a_{{1,2}}-ia_{{0,3}},\\
&\mathbi{e}_3(-3i)=a_{{3,0}}-3\,ia_{{2,1}}-3\,a_{{1,2}}+ia_{{0,3}}.
\end{align*}
}
\end{pr}


In  \cite{MP}, see also \cite{F2000}, the complex moments $c_{p,q}$ were introduced:
$$
c_{p,q}=\sum_{k=0}^p \sum_{j=0}^q \binom{p}{k} \binom{q}{j} (-1)^{q-j} \, i^{p+q-k-j} a_{k+j,p+q-k-j}.
$$
It is easy to see that  $c_{p,q}=\mathbi{e}_{p+q}((p-q) i).$
Thus the complex moments are  eigenvectors of  $D$ and can be expressed explicitly in terms of  Kravchuk polynomials:
\begin{gather*}
c_{p,q}=\sum_{j=0}^{p+q} i^j \mathcal{K}_j\left(q,p+q \right)a_{p+q-j,j}.
\end{gather*}



\section{Polynomial  rotation invariants.}


The problem of deriving  polynomial rotation invariants leads to  some Diophantine equation. 
Let us discuss the general framework and introduce the necessary terminology, see \cite{RS} for more details.

Let  $L$ be a diagonalizable derivation acting   on the polynomial algebra $ \mathbb{C}[e_1, e_2 \ldots, e_n]$ and let  $e_1, e_2 \ldots, e_n$ be its eigenvectors associated with the   eigenvalues $\lambda_1,\lambda_2,$ $ \ldots, \lambda_n$ . Then $L$ acts on a monomial in a such way 
\begin{gather*}
L \left(e_1^{k_1} e_2^{k_2} \cdots e_n^{k_n} \right)= L(e_1^{k_1}) e_2^{k_2} \cdots e_n^{k_n}+ e_1^{k_1} L(e_2^{k_2}) \cdots e_n^{k_n}+\cdots+e_1^{k_1} e_2^{k_2} \cdots L(e_n^{k_n})=\\= (k_1 \lambda_1+k_2 \lambda_2 \cdots+k_n \lambda_n) e_1^{k_1} e_2^{k_2} \cdots e_n^{k_n}, k_i \in \mathbb{N}
\end{gather*}
Thus each  non-negative integer solution of the equation 
 $k_1 \lambda_1+k_2 \lambda_2 \cdots+k_n \lambda_n=0$  defines the  kernel element  $e_1^{k_1} e_2^{k_2} \cdots e_n^{k_n}$ and we obtain:
$$
\ker L=\{e_1^{k_1} e_2^{k_2} \cdots e_n^{k_n} \mid k_1 \lambda_1+k_2 \lambda_2 \cdots+k_n \lambda_n=0   \}.
$$

Let us recall the notion of the \textit{Hilbert bases} of a monomial algebra.
Let 
$\mathbb{C}[x_1,x_2,\ldots,x_n]$ be the polynomial algebra over field $\mathbb{C}.$ A subalgebra   $\mathcal{M}\subset \mathbb{C}[x_1,x_2,\ldots,x_n] $ is called  a \textit{monomial algebra} if $\mathcal{M}$ can be generated  by monomials  $x_1^{k_1}x_2^{k_2}\cdots x_n^{k_n}$.  Assign to each monomial  $\mathbi{m}=x_1^{k_1}x_2^{k_2}\cdots x_n^{k_n}$ the  point  $p(\mathbi{m})=(k_1,k_2,\ldots, k_n)  $ in the lattice  $\mathbb{N}^n$. Put  $P(\mathcal{M})=\{ p(\mathbi{m}) \mid \mathbi{m} \in \mathcal{M} \}.$ The set  $P(\mathcal{M})$ forms a monoid.  By Gordan's lemma this monoid is finitely generated, see \cite{CL}. 
 The minimal spanning subset of $P(\mathcal{M})$ is called the \textit{Hilbert basis  } and denoted by  $H(\mathcal{M})$. Put  $G(\mathcal{M})=\{ \mathbi{m} \in \mathcal{M} \mid p(\mathbi{m}) \in H(\mathcal{M}) \}.$   Then the set  $G(\mathcal{M})$ is the minimal generating set of the monomial algebra  $\mathcal{M}.$ 
 

\begin{pr}{\rm  Let  a derivation   $L$ acts on the vector space $V=\langle e_1,e_2,e_3 \rangle:$ 
$$  
L(e_1)=e_1, L(e_2)=e_2, L(e_3)=-2e_3.
$$
The kernel of  $L$ is a monomial algebra
$$
\ker L=\{e_1^{k_1} e_2^{k_2} e_n^{k_3} \mid k_1+k_2 - 2 k_3=0   \}.
$$
We have  $P(\ker L)=\{ (k_1,k_2,k_3) \in \mathbb{N}^3  \mid k_1+k_2 - 2 k_3=0 \}.$
The vector  $\Phi=(1,1,-2)$ is called the  \textit{generating vector } of the monoid  $P(\ker L)$.

 It is easy to check that  $P(\ker L)$ is generated by the three solutions:  $(1,1,1),$ $(2,0,1)$ and  $(0,2,1)$. Therefore,  $H(\ker L)=\{ (1,1,1),(2,0,1),(0,2,1) \}$. Thus, the minimal generating set of  $P(\ker L)$  consists of the elements
$$
G(\ker L)=\{e_1e_2e_3,  e_1^2e_3, e_2^2e_3   \}.
$$

 }
\end{pr}

We now return to the derivarion  $D$. Since $D$ is diagonalizable   then   $\mathbb{C}[W_d]^{\mathfrak{so}_2}$ is a monomial algebrа. 
The following statement holds

\begin{te} The minimal generating set $G(\mathbb{C}[W_d]^{\mathfrak{so}_2}]), d \geq 2 $ of the algebra of rotation polynomial invariants  is as follows 

$$
G(\mathbb{C}[W_d]^{\mathfrak{so}_2}])=\{ \text{{\rm $\mathbi{m}$}} \mid p(\text{\rm $\mathbi{m}$})  \in H(\mathcal{M}_d) \},
$$
where 
 $\mathcal{M}_d$ is a monoid with the generating vector 
$$
\Phi_d=(\underbrace{1,\ldots,1}_{\text{$l_{1}$ times}}, \underbrace{2,\ldots,2}_{\text{$l_{2}$ times}}\ldots,\underbrace{d-1,\ldots,d-1}_{\text{$l_{d-1}$ times}},d,\underbrace{-1,\ldots,-1}_{\text{$l_{1}$ times}},\ldots,\underbrace{-(d-1),\ldots,-(d-1)}_{\text{$l_{d-1}$ times}},-d),
$$
 and  $$
l_s=\begin{cases} \left[ \frac{d}{2} \right], s=0,\\ d-\left[ \frac{d}{2} \right]-1, s=1,\\\left[ \frac{d-s}{2}\right]+1, \text{ $s>1$} 
\end{cases}.
$$

\end{te}
\begin{proof} Theorem~\ref{1} implies that the  operator  $D$ is diagonalizable on the direct sum  
$$
W^*_d=V^*_2 \oplus V^*_3 \oplus \cdots \oplus V^*_d.  
$$ 
with the spectrum 
$$
\Lambda=\Lambda_2 \cup \Lambda_3  \cup \cdots \cup \Lambda_d,
$$
where 
$$
\Lambda_n=\{{n i},{(n-2) i},\ldots {-n i}\}.
$$
Denote  by  $l_s$ the number of eigenvectors of $\Lambda$ with the eigenvalue $s i$. A simple combinatorial consideration shows that
$$
l_s=\begin{cases} \left[ \frac{d}{2} \right], s=0,\\ d-\left[ \frac{d}{2} \right]-1, s=1,\\\left[ \frac{d-s}{2}\right]+1, \text{ $s>1$}.
\end{cases}
$$
It is clear that  $l_d=1$  and   $l_{-s}=l_s$.

Consider the monomial  
$$
\mathbi{m}=\prod_{n=2}^{d} \prod_{\lambda \in \Lambda_n} \mathbi{e}_n^{k_n(\lambda)}(\lambda) \in  \mathbb{C}[W_d].
$$

Then
\begin{gather*}
D(\mathbi{m})=D\left( \prod_{n=2}^{d} \prod_{\lambda \in \Lambda_n} \mathbi{e}_n^{k_n(\lambda)}(\lambda)\right)=
\prod_{n=2}^{d} \prod_{\lambda \in \Lambda_n} \mathbi{e}_n^{k_n(\lambda)}(\lambda) \cdot \left( \sum_{n=2}^{d} \sum_{\lambda \in \Lambda_n }k_n(\lambda) \cdot  \lambda \right).
\end{gather*}
Since the eigenvalue  $\lambda=is$ occurs in $\Lambda$ exactly  $l_s$ times,  then  
\begin{gather*}
\sum_{n=2}^{d} \sum_{\lambda \in \Lambda_n }k_n(\lambda) \cdot  \lambda=\underbrace{(k_2(i)+k_3(i)+\cdots )}_{l_1 \text{  addition} } \cdot i +\underbrace{(k_2(2i)+ k_4(2i)+\cdots )}_{l_2 \text{ addition} } 2i+\cdots+\\+k_d(di) di +\cdots+k_d(-di) (-di).
\end{gather*}
Therefore,  a generating vector of the monoid  $\mathcal{M}_d$ has the form 
$$
(\underbrace{1,\ldots,1}_{\text{$l_{1}$ times}}, \underbrace{2,\ldots,2}_{\text{$l_{2}$ times}}\ldots,\underbrace{d-1,\ldots,d-1}_{\text{$l_{d-1}$ times}},d,\underbrace{-1,\ldots,-1}_{\text{$l_{1}$ times}},\ldots,\underbrace{-(d-1),\ldots,-(d-1)}_{\text{$l_{d-1}$ times}},-d).
$$
This proves the theorem.
\end{proof}

Let us consider some examples of calculating of  minimal generating sets.
\begin{pr}{\rm For  $d=2$ we have   $\Lambda_2= \{ 2i,0,-2i \}$  and   $l_1=l_0=1$. 
The generating vector is equal 
to  $(1,0,-1)$ and the  Hilbert basis is  
$$
(0,1,0), (1,0,1).
$$

Therefore the algebra $\mathbb{C}[W_2]^{\mathfrak{so}_2}$  is  generated  by the two monomials
$$
\mathbi{e}_2(0),\mathbi{e}_2(2i) \mathbi{e}_2(-2i).
$$
We have
\begin{align*}
&\mathbi{e}_2(0)=a_{{2,0}}+a_{{0,2}},\\
&\mathbi{e}_2(2i) \mathbi{e}_2(-2i)=\left( a_{{2,0}}+2\,i a_{{1,1}}-a_{{0,2}} \right)  \left( a_{{2,0}}-2\,ia_{{1,1}}-a_{{0,2}} \right)=(a_{2,0}-a_{0,2})^2+4 a_{1,1}^2.
\end{align*}

The corresponding moment invariants are the well known Hu's invariants. 

}

\end{pr}

\begin{pr}{\rm For $d=3$ 
we have 
$$
l_0=1,l_1=1,l_2=1,l_3=1.
$$
The generating vector is 
 $\Psi_3=(3,2,1,0,-1,-2,-3)$. To find a Hilbert basis we use a computer algebra system  
 CoCoA, see \cite{CA}. 
After computation we obtain the   Hilbert  basis  $H(\mathcal{M}_3):$
\begin{align*}
&[0, 0, 0, 1, 0, 0, 0], [0, 1, 0, 0, 0, 1, 0], [1, 0, 0, 0, 0, 0, 1], [0, 0, 1, 0, 1, 0, 0], [1, 0, 0, 0, 1, 1, 0],\\ &[0, 1, 1, 0, 0, 0, 1], [0, 0, 2, 0, 0, 1, 0], [0, 1, 0, 0, 2, 0, 0], [1, 0, 1, 0, 0, 2, 0], [0, 2, 0, 0, 1, 0, 1],\\ &[1, 0, 0, 0, 3, 0, 0], [0, 0, 3, 0, 0, 0, 1], [2, 0, 0, 0, 0, 3, 0], [0, 3, 0, 0, 0, 0, 2].
 \end{align*}

Assign to each such  vector  an element of the minimal generating set.  For instance, $$
p([1, 0, 0, 0, 1, 1, 0])={\mathbi{e}_3(3i)\mathbi{e}_3(-i)\mathbi{e}_2(-2i)}.$$ 
Put  $$x_0=\mathbi{e}_2(0),x_1=\mathbi{e}_3(i),x_2=\mathbi{e}_2(2i),x_3=\mathbi{e}_3(3i),y_1=\mathbi{e}_3(-i),y_2=\mathbi{e}_2(-2i),y_3=\mathbi{e}_3(-3i).$$
Thus, a minimal generating set of the algebra of polynomial rotation invariants of order 3 consists of   14  elements: 
\begin{center}
\begin{tabular}{|l|l|l|}
\hline
Degree & Invariants & \#\\
\hline
1      &  $ x_0 $  &1      \\
\hline
2      &   $x_1 y_1, x_2 y_2,x_3 y_3$  &3      \\
\hline
3      &    $ x_{{3}}y_{{1}}y_{{2}}, x_{{1}}x_{{2}}y_{{3}},{x_{{1}}}^{2}y_{{2}},x_{{2}}{y_{{1}}}^{2\phantom{2^{b}}}$      & 4 \\
\hline
4      &    $ x_{{1}}x_{{3}}{y_{{2}}}^{2},{x_{{2}}}^{2}y_{{1}}y_{{3}},x_{{3}}{y_{{1}}}^{3\phantom{2^{b}}}, {x_{{1}}}^{3}y_{{3}} $    & 4  \\
\hline
5      &    ${x_{{3}}}^{2}{y_{{2}}}^{3},{x_{{2}}}^{3}{y_{{3}}}^{2\phantom{2^{b}}}$ & 2 \\
\hline
\multicolumn{2}{|l|}{Total} &14 \\
\hline      
\end{tabular}
\end{center}

For the Hu invariant we have
\begin{gather*}
h_5=\frac{1}{2}({x_{{1}}}^{3}y_{{3}}+x_{{3}}{y_{{1}}}^{3}),
h_6=\frac{1}{2} ({x_{{1}}}^{2}y_{{2}}+x_{{2}}{y_{{1}}}^{2}),
h_7=-\frac{i}{2} \left(x_{{3}}{y_{{1}}}^{3}-{x_{{1}}}^{3}y_{{3}} \right)
\end{gather*}
}
\end{pr}

\begin{pr}{\rm 
For the case  $d=4$ we have 
$$
l_{{0}}=2,l_{{1}}=1,l_{{2}}=2,l_{{3}}=1,l_{{4}}=1.
$$
and 
$$
\Phi_4=(1,2,2,3,4,0,0,-1,-2,-2,-3,-4).
$$
The Hilbert basis ( СоСоA) consists of 65 elements.
For convenience, we denote $x_{kj}=\mathbi{e}_k(j i),$ $ y_{k,j}=\mathbi{e}_k(-j  i)$  for  $j>0$.
A minimal generating set of the algebra of polynomial rotation invariants of order 4 shown   below 
\begin{center}
\begin{tabular}{|l|l|l|}
\hline
{\rm deg} & Invariants &\#\\
\hline
1      &  $ x_{20},x_{40} $ &2       \\
\hline
2      &   $x_{{33}}y_{{33}},x_{{44}}y_{{44}},x_{{42}}y_{{22}},x_{{42}}y_{{42}},x_{{22}}y_{{22}},x_{{22}}y_{{42}},x_{{31}}y_{{31}}$ &7       \\
\hline
3      &    $x_{{44}}y_{{31}}y_{{33}},x_{{42}}{y_{{31}}}^{2},x_{{33}}y_{{31}}y_{{42}},x_{{33}}y_{{31}}y_{{22}},x_{{3
1}}x_{{42}}y_{{33}}$  &  16   \\

 & $x_{{31}}x_{{22}}y_{{33}},x_{{31}}x_{{33}}y_{{44}},{x_{{42}}}^{2}y_{{44}},x_{{22}}x_{{42}}y_{{44}},x_{{44}}y_{{22}}y_{{42}}$ & \\

 & $x_{{44}}{y_{{42}}}^{2},{x_{{31}}}^{2}y_{{22}},{x_{{31}}}^{2}y_{{42}},{x_{{22}}}^{2}y_{{44}},x_{{22}}{y_{{31}}}^{2},x_
{{44}}{y_{{22}}}^{2\phantom{2^{b}}}$ & \\
\hline
4      &    ${x_{{31}}}^{2}x_{{42}}y_{{44}},x_{{31}}x_{{44}}y_{{42}}y_{{33}},{x_{{31}}}^{2}x_{{22}}y_{{44}},x_{{22}}x_{{33}}y_{{31}}y_{{44
}},x_{{31}}x_{{33}}{y_{{42}}}^{2}$ & 20    \\
      &    $x_{{31}}x_{{33}}{y_{{22}}}^{2},{x_{{22}}}^{2}y_{{31}}y_{{33}},{x_{{42}}}^{2}y_{{31}}y_{{33}},x_{{22}}x_{{42}}y_{{31}}y_{{33}},x_{{31}}x_{{33}}y_{{22}}y_{{4
2}}
$ &  \\
 & $x_{{22}}x_{{44}}{y_{{33}}}^{2\phantom{2^{b}}},x_{{42}}x_{{33}}y_{{31}}y_{{44}},x_{{31}}x_{{44}}y_{{22}}y_{{33}},x_{{44}}{y_{{31}}}^{2}y_{{22
}},x_{{44}}{y_{{31}}}^{2}y_{{42}}$ & \\
 & $
{x_{{33}}}^{2}y_{{42}}y_{{44}},{x_{{33}}}^{2}y_{{22}}y_{{44}},x_
{{42}}x_{{44}}{y_{{33}}}^{2},x_{{33}}{y_{{31}}}^{3},{x_{{31
}}}^{3}y_{{33}}$ & \\
\hline
5 & ${x_{{44}}}^{2}y_{{22}}{y_{{33}}}^{2},{x_{{44}}}^{2}y_{{42}}{y_{{33}}}^{2},{x_{{31}}}^{2}x_{{44}}{y_{{33}}}^{2},x_{{22}}{x_{{33}}
}^{2}{y_{{44}}}^{2},x_{{42}}{x_{{33}}}^{2}{y_{{44}}}^{2}$ & 16\\
 & ${x_{{22}}}^{3}{y_{{33}}}^{2},{x_{{33}}}^{2}{y_{{42}}}^{3},{x_{{33}}}^{2}y_{{22}}{y_{{42}}}^{2},{x_{{33}}}^{2}{y_{{22}}}^{2}y_{{42
}},{x_{{33}}}^{2}{y_{{22}}}^{3},{x_{{42}}}^{3}{y_{{33}}}^{2}$ & \\
 & $x_{{22}}{x_{{42}}}^{2}{y_{{33}}}^{2},{x_{{22}}}^{2}x_{{42}}{y_{{33}}}^{2},x_{{44}}{y_{{31}}}^{4},{x_{{31}}}^{4}y_{{44}},{x_{{33}
}}^{2}{y_{{31}}}^{2\phantom{2^{b}}}\! \! \!y_{{44}}$ & \\
\hline
6 & $x_{{31}}{x_{{44}}}^{2}{y_{{33}}}^{3},{x_{{33}}}^{3}y_{{31}}{y_{{44}}}^{2\phantom{2^{b}}}$ & 2 \\
\hline
7 & ${x_{{44}}}^{3}{y_{{33}}}^{4},{x_{{33}}}^{4}{y_{{44}}}^{3\phantom{2^{b}}}$ & 2 \\
\hline  
    \multicolumn{2}{|l|}{Total} &65 \\
\hline 
\end{tabular}
\end{center}
}
\end{pr}


Proceeding in the same manner as above,  we obtain that a minimal generating set of the algebra $\mathbb{C}[W_5]^{\mathfrak{so}_2}$  consists of  562 invariants up to degree 9.


\section{Analogue of  Cayley-Sylvester formula and Poincare series.}


The number of linearly independent homogeneous $\mathfrak{sl}_2$-invariants of  binary form of  given degree    is  calculated  by the well-known  Cayley-Sylvester formula, see  \cite{Sylv}, \cite{Hilb}. In the section we find a  similar formula for  $\mathfrak{so}_2$-invariants.

The formal sum 
$$
{\rm Char}(W_d)=\sum_{k=-d}^d l_k q^k,
$$
is called the \textit{character} of  $\mathfrak{so_{2}}$  on $W_d$.    
Here  $q$ is a formal parameter. 

\begin{pr}{\rm For   $d=4$ we have  $l_0=2, l_1=1, l_2=2, l_3=1, l_4=1,$ and  
$$
{\rm Char}(W_4)=q^{-4}+q^{-3}+2 q^{-2}+q^{-1}+2+q+2q^2+q^3+q^4.
$$
}
\end{pr}

The following statement holds.

\begin{te} 
The number of linearly independent homogeneous  invariants of $\mathbb{C}[W_d]^{\mathfrak{so}_2}$ of degree  $n$   is equal to the number of   non-negative integer solutions of the system of equations: 
$$
\begin{cases} \displaystyle 
\sum_{k=1}^{2d}k(\alpha_{1}^{(k)}+\alpha_{2}^{(k)}+\cdots+\alpha_{l_k}^{(k)})=dn,\\
\displaystyle  \sum_{k=0}^{2d}(\alpha_{1}^{(k)}+\alpha_{2}^{(k)}+\cdots+\alpha_{l_k}^{(k)})=n.
\end{cases}
$$
\end{te}

\begin{proof} Let $S(W_d)$  be the symmetrical algebra of the vector space $W_d^*.$ 
Then 
$$
S(W_d)=S^0(W_d)+S^1(W_d)+\cdots +S^n(W_d)+\cdots,
$$
where  $S^n(W_d)$  is the vector space of homogeneous polynomials of degree  $n.$

We have that 
$$
{\rm Char}(S^n(W_d)) = \sum_{k=-d n}^{dn}\gamma_d(n,k)q^k,  
$$
where $\gamma_d(n,k)$ is the dimension of the vector space generated by all eigenvectors with the eigenvalue $k i$. 

Since all $\mathfrak{so_{2}}$-invariants have zero eigenvalue then    $\dim S^n(\mathbb{C}[W_d]^{\mathfrak{so}_2})=\gamma_d(n,0), $ where $$S^n(\mathbb{C}[W_d]^{\mathfrak{so}_2})=S^n(W_d) \cap \mathbb{C}[W_d]^{\mathfrak{so}_2}.$$
From another side, the character   $ {\rm Char}(S^n(W_d))$   equals   $H_d(q^{-d},q^{-d+2},\ldots,q^{d}),$ see \cite{FH},  where   $H_d(x_0,x_1,\ldots,x_m)$ is the complete symmetrical polynomial   $$H_n(x_0,x_1,\ldots,x_m)=\sum_{|\alpha|=n} x_0^{\alpha_0} x_1^{\alpha_1} \ldots x_m^{\alpha_m}, |\alpha|=\sum_{j=0}^m \alpha_j.$$

Then we have  
\begin{gather*}
{\rm Char}(S^n(W_d))= H_n\bigl(\underbrace{q^{-d}, \ldots,q^{-d} }_{\text{$l_d$ times}},\ldots,\underbrace{q^{0}, \ldots,q^{0} }_{\text{$l_0$ times}}, \ldots, \underbrace{q^{d}, \ldots,q^{d} }_{\text{$l_d$ times}}\bigr)=\\
=\sum_{|\alpha|=n}\prod_{k=0}^{2d}q^{(d-k)(\alpha_{1}^{(k)}+\alpha_{2}^{(k)}+\cdots+\alpha_{l_{d-k}}^{(k)})}=\sum_{|\alpha|=n}q^{\sum\limits_{k=0}^{2d}(d-k)(\alpha_{1}^{(k)}+\alpha_{2}^{(k)}+\cdots+\alpha_{l_{d-k}}^{(k)})}=\\=
\sum_{|\alpha|=n}q^{d\sum\limits_{k=0}^{2d}(\alpha_{1}^{(k)}+\alpha_{2}^{(k)}+\cdots+\alpha_{l_{d-k}}^{(k)})-\sum\limits_{k=0}^{2d}k(\alpha_{1}^{(k)}+\alpha_{2}^{(k)}+\cdots+\alpha_{l_{d-k}}^{(k)})}=\sum_{|\alpha|=n}q^{dn-\sum\limits_{k=0}^{2d}k(\alpha_{1}^{(k)}+\alpha_{2}^{(k)}+\cdots+\alpha_{l_{d-k}}^{(k)})}=
\\=\sum_{k=0}^{2d\,n} \omega_d(n,s) q^{d\,n-s},
\end{gather*}
here  $\omega_d(n,s) $  is the  number of   non-negative integer solutions of the system of equations:
$$
\begin{cases} \displaystyle 
\sum_{k=1}^{2d}k(\alpha_{1}^{(k)}+\alpha_{2}^{(k)}+\cdots+\alpha_{l_{d-k}}^{(k)})=s,\\
\displaystyle  \sum_{k=0}^{2d}(\alpha_{1}^{(k)}+\alpha_{2}^{(k)}+\cdots+\alpha_{l_{d-k}}^{(k)})=n.
\end{cases}
$$ 
 Particularly, the coefficient of  $q^0$ is equal  $ \omega_d(n,d\,n).$
Thus  $\dim S^n(\mathbb{C}[W_d]^{\mathfrak{so}_2})= \omega_d(n,d\,n)$  as required.
\end{proof}

Let us derive a more convenient formula for calculation   $\dim S^n(\mathbb{C}[W_d]^{\mathfrak{so}_2})$.

The ordinary generating function  
$$
\mathcal{P}_d(z)=\sum_{n=0}^\infty \dim S^n(\mathbb{C}[W_d]^{\mathfrak{so}_2})  z^n,
$$ 
of the sequence  $\{ S^n(\mathbb{C}[W_d]^{\mathfrak{so}_2}) \}$  is called  the \textit{Poincar\'e series} of the algebra $\mathbb{C}[W_d]^{\mathfrak{so}_2}.$
The Poincar\'e series can be expressed as a  contour integral.
\begin{te}$$
\displaystyle \mathcal{P}_{d}(z){=}\frac{1}{2\pi i} \oint_{|t|=1} \frac{1}{\prod\limits_{i=0}^{2d}(1-zt^{i-d})^{l_{i}} } \frac{dt}{t}.
$$
\end{te}
\begin{proof}
It is well known   that  the  number of   non-negative integer solutions of the system of equations
$$
\begin{cases} \displaystyle 
\sum_{k=1}^{2d}k(\alpha_{1}^{(k)}+\alpha_{2}^{(k)}+\cdots+\alpha_{l_k}^{(k)})=dn,\\
\displaystyle  \sum_{k=0}^{2d}(\alpha_{1}^{(k)}+\alpha_{2}^{(k)}+\cdots+\alpha_{l_k}^{(k)})=n,
\end{cases}
$$
equal to the coefficient of  $\displaystyle t^n z^{dn} $ of the expansion of the series 
$$
f_{d}(t,z)=\frac{1}{(1-t )^{l_d}(1-t\,z)^{l_{d-1}}(1-tz^2)^{ l_{d-2}}(1-tz^3)^{l_{d-3}}\ldots (1-t\,z^{2d})^{ l_d}}.
$$
Denote it in a such way
$$
\dim S^n(\mathbb{C}[W_d]^{\mathfrak{so}_2})=\left[ (t z^d)^n\right]f_{d}(t,z).
$$
Then by using the  Cauchy integral formula  we get
\begin{gather*}
\mathcal{P}_d(z)=\sum_{n=0}^\infty \dim S^n(\mathbb{C}[W_d]^{\mathfrak{so}_2})  z^n=\sum_{n=0}^\infty \left(\left[ (t z^d)^n\right]f_{d}(t,z) \right)  z^n=\sum_{n=0}^\infty \left(\left[ t^n\right]f_{d}(t z^{-d},z) \right) =\\=
 \sum_{n=0}^\infty \left(\left[ t^n\right]\frac{1}{2\pi i} \oint_{|t|=1} f_{d}(t z^{-d},z) \frac{dt}{t}\right)  z^n=\frac{1}{2\pi i} \oint_{|t|=1} f_{d}(t z^{-d},z) \frac{dt}{t}.
\end{gather*}
\end{proof}

\begin{pr}{\rm For  $d=3$ we have   $l_0=l_1=l_2=l_3=1$  and

$$
f_3(t,z)={\frac {1}{ \left( 1-t \right)  \left( 1-tz \right)  \left( 1-t{z}^{2
}\right)  \left( 1-t{z}^{3} \right)  \left( 1-t{z}^{4} \right) 
 \left( 1-t{z}^{5} \right)  \left( 1-t{z}^{6}\right) }}.
$$

After computing   the corresponding integral we obtain 
$$
\mathcal{P}_d(z)={\frac {{z}^{10}+{z}^{8}+3\,{z}^{7}+4\,{z}^{6}+4\,{z}^{5}+4\,{z}^{4}+3
\,{z}^{3}+{z}^{2}+1}{ \left( {z}^{3}-1 \right)  \left( {z}^{5}-1
 \right)  \left( {z}^{4}-1 \right)  \left( {z}^{2}-1 \right) ^{2}
 \left( z-1 \right) }}.
$$
Expand it in a series  we have 
$$
\mathcal{P}_3(z)=1+z+4\,{z}^{2}+8\,{z}^{3}+18\,{z}^{4}+32\,{z}^{5}+58\,{z}^{6}+94\,{z}
^{7}+151\,{z}^{8}+227\,{z}^{9}+\cdots
$$
Therefore, the algebra  $\mathbb{C}[W_3]^{\mathfrak{so}_2}$ consists one invariant of degree 1, namely $x_0$ in the notation of   Example~4.2. Also, there exists 4 linearly independent invariants of degree  2, namely  $x_1y_1,x_2y_2,x_3y_3, x_0^2,$  4 linearly independent invariants of degree  3 and so on.
}
\end{pr}

In the same way we obtain

\begin{gather*}
\mathcal{P}_4(z)=\frac {p_4(z)}{
 \left( 1-{z}^{3} \right) ^{3} \left( 1-{z}^{5} \right) ^{2} \left( 1-{z}^{7} \right)  \left( 1-{z}^{4} \right)  \left( 1-{z}^{2}
 \right) ^{2} \left( 1-z \right) ^{2}
}=\\=
1+2\,z+10\,{z}^{2}+34\,{z}^{3}+105\,{z}^{4}+288\,{z}^{5}+720\,{z}^{6}+1660\,{z}^{7}+3588\,{z}^{8}+7326\,{z}^{9}+\cdots, 
\end{gather*}
where
\begin{gather*}
p_4(z)={z}^{24}+5\,{z}^{22}+13\,{z}^{21}+33\,{z}^{20}+63\,{z}^{19}+112\,{z}^{
18}+174\,{z}^{17}+252\,{z}^{16}+331\,{z}^{15}+400\,{z}^{14}+\\+445\,{z}^{
13}+464\,{z}^{12}+445\,{z}^{11}+400\,{z}^{10}+331\,{z}^{9}+252\,{z}^{8
}+174\,{z}^{7}+112\,{z}^{6}+63\,{z}^{5}+33\,{z}^{4}+\\+13\,{z}^{3}+5\,{z}
^{2}+1.
\end{gather*}

\section{The algebra of rational invariants  $\mathbb{C}(W_d)^{\mathfrak{so}_2}.$}


 The  rational invariants are more interested in  applications then polynomial invariants.

	In the following theorem we find the size of a minimal generating  set of the algebra of rational rotation invariants. 
\begin{te} The number of element in a minimal generating set of the algebra   $\mathbb{C}(W_d)^{\mathfrak{so}_2}$  equals 
$$
\frac{\left( d+4 \right)  \left( d-1 \right)}{2}-1.
$$
\end{te}
\begin{proof}

Since the group  $SO(2)$ as an affine variety is one-dimensional, then, see \cite{SP}, the transcendence degree of  the field  extension  $\mathbb{C}(W_d)^{\mathfrak{so}_2}/\mathbb{C}$ equals 
$$
{\rm tr\,deg}_{\mathbb{C}} \mathbb{C}(W_d)^{\mathfrak{so}_2}= \dim W_d- \dim SO(2).
$$
Thus, the algebra  $\mathbb{C}(W_d)^{\mathfrak{so}_2}$ consists of exactly   $\dim W_d- 1$ algebraically independent elements. Taking into account that 
$$
\dim W_d=\sum_{k=2}^d \dim V_k=\sum_{k=2}^d (k+1)=\frac{\left( d+4 \right)  \left( d-1 \right)}{2},
$$
which is what had to be proved.
\end{proof}

An explicit form of the rational invariants can be found as solutions of some partial differential equation.
  In the following theorem we present  a minimal generating set of the algebra $\mathbb{C}(W_d)^{\mathfrak{so}_2}.$ 
\begin{te} Let  $ q i \in \Lambda_p, q >0 $ be a fixed eigenvalue of the derivative $D$. The the set of  $\dim W_d -1$  invariants 
$$
G^{(d)}_{p,q}=\left\{\textbf{e}_{2j}(0), \textbf{e}_{n}(si)\textbf{e}_{n}(-si), \textbf{e}_n(si)^q{e_p(-q i)^{s}} \mid 2 \leq n \leq d, j \leq l_0, s i \in \Lambda_n,  s >0, q \neq s \right\},
$$
forms a transcendence basis of the algebra invariants $\mathbb{C}(W_d)^{\mathfrak{so}_2}$.
\end{te}
\begin{proof} 

Let us consider  an  auxiliary derivation  $\mathcal{D}$ on  the polynomial field in the indeterminates $x_1,y_1,x_2,y_2,\ldots,x_n,y_n,z_1,z_2,\ldots,z_l$ with the following action 
$$
\mathcal{D}(x_j)=\lambda_j x_j, \mathcal{D} (y_j)=-\lambda_j y_j, \mathcal{D}(z_j)=0 \lambda_j \neq 0.
$$
Then its  kernel is exactly the same  as  solutions   $$u=u(x_1,x_2,\ldots,x_n,y_1,\ldots,y_n,z_1,z_2,\ldots,z_l),$$ of the partial differential equation 
$$
\lambda_1 x_1 \frac{\partial u}{\partial x_1}+\lambda_2 x_2 \frac{\partial u}{\partial x_2}+\cdots+\lambda_n x_n \frac{\partial u}{\partial x_n}-\lambda_1 y_1 \frac{\partial u}{\partial y_1}-\cdots-\lambda_n y_n \frac{\partial u}{\partial y_n}=0.
$$
As is well known, these  solutions are  first integrals of the system of differential equations
$$
\begin{cases}
\displaystyle \frac{dx_1}{\lambda_1 x_1}=\frac{dx_2}{\lambda_2 x_2}=\cdots=\frac{dx_n}{\lambda_n x_n}=-\frac{dy_1}{\lambda_1 y_1}=\cdots=-\frac{dy_n}{\lambda_n y_n},\\
d z_1=dz_2=\cdots=dz_l=0.
\end{cases}
$$
This system is easy to solve: 
$$
x_j=C'_j e^{\lambda_j t}, y_j=C''_j e^{-\lambda_j t}, z_k=C'''_j, i=1, \ldots,n, k=1, \ldots, l,
$$
for some constants $C'_j,C''_j,C'''_j.$
We eliminate the parameter  $t$ and get  $2n+l-1$ first integrals:
$$
x_j y_j=C_j, x_k^{\lambda_q}y_q^{\lambda_k}=C_{n-1+k}, z_m=C_{2n-1+m}, j,k \leq  n, k \neq q, m\leq l.
$$
Let us prove that the first integrals are algebraically  independent. 
We may assume, without loss of generality, that $q=1.$  We  need to check, see \cite{Kam}, that the rank of the jacobian $(2n \times 2n-1)$-matrix
$$
J=\begin{pmatrix}
y_{{1}}&0&0&0&\ldots&x_{{1}}&0&0&0\\ 
0&y_{{2}}&0&0&\ldots&0&x_{{2}}&0&0\\
0&0&y_{{3}}&0&\ldots&0&0&x_{{3}}&0\\
\ldots&\ldots&\ldots&\ldots&\ldots&\ldots&\ldots&\ldots&\ldots\\
0&0&\ldots&y_{{n}}&\ldots&0&0&\ldots&x_{{n}}\\
 0&\lambda_1 x_2^{\lambda_1-1}{y_{{1}}}^{\lambda_2}&0&0&\ldots&\lambda_2x_{{2}}^{\lambda_1}{y_{{1}}^{\lambda_2-1}}&0&\ldots&0\\ 
0&0&\lambda_1 x_3^{\lambda_1-1}{y_{{1}}}^{\lambda_3}&0&\ldots&\lambda_3 x_{{3}}^{\lambda_1}{y_{{1}}^{\lambda_3-1}}&0&\ldots&0\\
\ldots&\ldots&\ldots&\ldots&\ldots&\ldots&\ldots&\ldots &\ldots\\
0&0&\ldots&\lambda_1 x_n^{\lambda_1-1}{y_{{1}}}^{\lambda_n}&\ldots&\lambda_n x_{{n}}^{\lambda_1}{y_{{1}}^{\lambda_n-1}}&0&\ldots&0
\end{pmatrix}
$$
 equals  $2n-1.$  
Let us compute the  minor formed  by deleting the  $n+1$-th  column of $J.$ We have 
\begin{gather*}
\begin{vmatrix}
y_{{1}}&0&0&0&\ldots&0&0&0\\ 
0&y_{{2}}&0&0&\ldots&x_{{2}}&0&0\\
0&0&y_{{3}}&0&\ldots&0&x_{{3}}&0\\
\ldots&\ldots&\ldots&\ldots&\ldots&\ldots&\ldots&\ldots\\
0&0&\ldots&y_{{n}}&0&\ldots&0&x_{{n}}\\
 0&\lambda_1 x_2^{\lambda_1-1}{y_{{1}}}^{\lambda_2}&0&0&\ldots&0&\ldots&0\\ 
0&0&\lambda_1 x_3^{\lambda_1-1}{y_{{1}}}^{\lambda_3}&0&\ldots&0&\ldots&0\\
\ldots&\ldots&\ldots&\ldots&\ldots&\ldots&\ldots &\ldots\\
0&0&\ldots&\lambda_1 x_n^{\lambda_1-1}{y_{{1}}}^{\lambda_n}&\ldots&0&\ldots&0
\end{vmatrix}=\\=y_1 \begin{vmatrix}
y_{{2}}&0&0&\ldots&x_{{2}}&0&0\\
0&y_{{3}}&0&\ldots&0&x_{{3}}&0\\
\ldots&\ldots&\ldots&\ldots&\ldots&\ldots&\ldots\\
0&\ldots&y_{{n}}&0&\ldots&0&x_{{n}}\\
\lambda_1 x_2^{\lambda_1-1}{y_{{1}}}^{\lambda_2}&0&0&\ldots&0&\ldots&0\\ 
0&\lambda_1 x_3^{\lambda_1-1}{y_{{1}}}^{\lambda_3}&0&\ldots&0&\ldots&0\\
\ldots&\ldots&\ldots&\ldots&\ldots&\ldots &\ldots\\
0&\ldots&\lambda_1 x_n^{\lambda_1-1}{y_{{1}}}^{\lambda_n}&\ldots&0&\ldots&0
\end{vmatrix}=\\=
-y_1 \begin{vmatrix}
\lambda_1 x_2^{\lambda_1-1}{y_{{1}}}^{\lambda_2}&0&\ldots&0&\\
0&\lambda_1 x_3^{\lambda_1-1}{y_{{1}}}^{\lambda_3}&\ldots&0&\\
\ldots&\ldots&\ldots&\ldots\\
0&\ldots&0&\lambda_1 x_n^{\lambda_1-1}{y_{{1}}}^{\lambda_n}
\end{vmatrix}  \cdot \begin{vmatrix}
x_2&0&\ldots&0&\\
0&x_{3}&\ldots&0&\\
\ldots&\ldots&\ldots&\ldots\\
0&\ldots&0&x_{n}
\end{vmatrix} \neq 0.
\end{gather*}

Since  the rank of the matrix  $J$ equal  $2n-1$, then the following   $2n-1$ first integrals  
 $$
x_j y_j, x_k^{\lambda_q}y_q^{\lambda_k} j,k \leq  n, k \neq q.
$$
are algebraically independent.

Now, since the derivation  $D$ acts on $W_d$ in the same way as the derivation $\mathcal{D}$   
$$
D(\mathbi{e}_n(is))= is \mathbi{e}_n(is), D(\mathbi{e}_n(-is))= -is \mathbi{e}_n(is),D(\mathbi{e}_n(0))=0, i\,s \in \Lambda_n, n=2,3,\ldots,d,
$$
then,  taking into account  that all eigenvectors are linearly independent, we get that the set 
$$
G^{(d)}_{p,q}=\left\{\mathbi{e}_{2j}(0), \mathbi{e}_{n}(si)\mathbi{e}_{n}(-si), \mathbi{e}_n(si)^q{e_p(-q i)^{s}} \mid 2 \leq n \leq d, j \leq l_0, s i \in \Lambda_n,  s >0, q \neq s \right\},
$$
is a minimal generating set of the algebra rotation invariants  $\mathbb{C}(W_d)^{\mathfrak{so}_2}.$

\end{proof}

Note, that the result confirms the   result of Flusser \cite{FF}.

Consider some examples.

\begin{pr}{\rm  Let  $d=4$ and put $x_{nk}=e_n(ki)$ and  $y_{nk}=e_n(-ki)$  for  $k \geq 0.$
Put  $p=3,q=1.$
Then the minimal generating set $G^{(4)}_{3,1}$ of  the rotation  invariants of order 4 has the form

\begin{center}
\begin{tabular}{|l|l|l|}
\hline
Degree & Invariants & \#\\
\hline
1      &  $ \beta_{1}=x_{20}, \beta_2=x_{40}  $  &2      \\
\hline
2      &   $\beta_3 =x_{22}y_{22}$, $\beta_4 =x_{31}y_{31}$ $\beta_5 =x_{33}y_{33}$, $\beta_6 =x_{42}y_{42}$ $\beta_7 =x_{44}y_{44}$&5      \\
\hline
3      &    $ \beta_8=x_{22} y_{31}^2, \beta_{9}=x_{42} y_{31}^2,   {\phantom{2^{b}}}$      & 2\\
 \hline
4      &   $\beta_{10}=x_{33} y_{31}^3$      & 1  \\
\hline
5      &    $\beta_{11}=x_{44} y_{31}^4 {\phantom{2^{b}}}$ & 1 \\
\hline
\multicolumn{2}{|l|}{Total} &11 \\
\hline      
\end{tabular}
\end{center}

In \cite{F2000} another minimal generating set of $\mathbb{C}(W_3)^{\mathfrak{so}_2}$ are presented in terms of the complex moments $c_{p,q}$:
\begin{gather*}
\phi_1=c_{1,1},
\phi_2=c_{2,1}c_{1,2},
\phi_3=Re(c_{2,0} c_{1,2}^2),
\phi_4=Im(c_{2,0}c_{1,2}^2),
\phi_5=Re(c_{3,0}c_{1,2}^3),
\phi_6=Im(c_{3,0}c_{1,2}^3),\\
\phi_7=c_{2,2}, 
\phi_8=Re(c_{3,1}c_{1,2}^2),
\phi_9=Im(c_{3,1}c_{1,2}^2),
\phi_{10}=Re(c_{4,0}c_{1,2}^4),
\phi_{11}=Im(c_{4,0}c_{1,2}^4).
\end{gather*}

We can express these invariants in terms of the invariants  $\beta_j:$
\begin{align*}
&\phi_{{1}}=\beta_{{1}},\phi_{{2}}=\beta_{{4}},\phi_{{3}}=\frac{1}{2}\,{\frac {\beta_{{3}}{\beta_{{4}}}^{2}+{\beta_{{8}}}^{2}
}{\beta_{{8}}}},\phi_{{4}}={\frac{i}{2} \frac{\beta_{{3}}{\beta_{{4}}}^{2}-{\beta_{{8}}}^{2} }{\beta_{{8}}}}, \phi_{{5}}=\frac{1}{2}\,{\frac {\beta_{{5}}{\beta_{{4}}}^{3}+{\beta_{{10}}}^{2
}}{\beta_{{10}}}},\\
&\phi_{{6}}=\frac{i}{2}{\frac { \beta_{{5}}{\beta_{{4}}}^{3}-{\beta_{{10
}}}^{2} }{\beta_{{10}}}}, \phi_{{7}}=\beta_{{2}}, \phi_{{8}}=\frac{1}{2}\,{\frac {\beta_{{6}}{\beta_{{4}}}^{2}+{\beta_{{9}}}^{2}}{\beta_{{9}}}},\phi_{{9}}=-\frac{i}{2}{\frac { \beta_{{6}}{\beta_{{4}}}^{2}-{\beta_{{9
}}}^{2} }{\beta_{{9}}}},\\
&\phi_{{10}}=\frac{1}{2}\,{\frac {\beta_{{7}}{\beta_{{4}}}^{4}+{\beta_{{11}}}^{2}}{\beta_{{11}}}}, \phi_{{11}}=-\frac{i}{2}{\frac { \beta_{{7}}{\beta_{{4}}}^{4}-{\beta_{{
11}}}^{2}  }{\beta_{{11}}}}.
\end{align*}

}

\end{pr}

\begin{pr}{\rm  Let  $d=5$. Put  $x_{nk}=e_n(ki)$, $y_{nk}=e_n(-ki)$,  $k \geq 0$ and 
  $p=3,q=1.$
Then the  minimal generating system  $G^{(5)}_{3,1}$ consists of  17 elements:

\begin{center}
\begin{tabular}{|l|l|l|}
\hline
Degree & Invariants & \#\\
\hline
1      &  $ \beta_{1}=x_{20}, \beta_2=x_{40}  $  &2      \\
\hline
2      &   $\beta_3 =x_{22}y_{22}$, $\beta_4 =x_{31}y_{31}$ $\beta_5 =x_{33}y_{33}$, $\beta_6 =x_{42}y_{42}$ $\beta_7 =x_{44}y_{44}$&9     \\
& $\beta_8=x_{55}y_{55}, \beta_9=x_{53}y_{53}, \beta_{10}=x_{51}y_{51}, \beta_{11}=x_{51}y_{31}$   & \\
\hline
3      &    $ \beta_{12}=x_{22} y_{31}^2, \beta_{13}=x_{42} y_{31}^2  {\phantom{2^{b}}}$      & 2\\
 \hline
4      &   $\beta_{14}=x_{33} y_{31}^3, \beta_{15}=x_{53}y_{31}^3$      & 2  \\
\hline
5      &    $\beta_{16}=x_{44} y_{31}^4 {\phantom{2^{b}}}$ & 1 \\
\hline
6     &    $\beta_{17}=x_{55} y_{31}^5 {\phantom{2^{b}}}$ & 1 \\
\hline
\multicolumn{2}{|l|}{Total} &17 \\
\hline      
\end{tabular}
\end{center}

}

\end{pr}

The invariants of the first two types  $\mathbi{e}_{2j}(0),$ and $ \mathbi{e}_{n}(si)\mathbi{e}_{n}(-si)$  can be written explicitly.

\begin{te} 
$$
\begin{array}{ll}
(i) & \displaystyle \textbf{e}_{2j}(0)=\sum_{k=0}^{j} \binom{j}{k} a_{2j-2k, 2k}, 2j \leq d,\\
(ii) & \textbf{e}_{n}({si})\textbf{e}_{n}(-{si})=
\displaystyle \left(\sum_{j=0}^{\left[\frac{n}{2}\right]} (-1)^j \mathcal{K}_{2j}\left(\frac{1}{2}(n-s),n \right)a_{n-2j,2j}\right)^2+\\  &+\displaystyle \left( \sum_{j=0}^{\left[\frac{n}{2}\right]} (-1)^j \mathcal{K}_{2j+1}\left(\frac{1}{2}(n-s),n \right)a_{n-2j-1,2j+1} \right)^2.
\end{array}
$$
\end{te} 
\begin{proof}
$(i)$ 
The   alternating Vandermonde's convolution formula, see \cite{G}, identity 3.32
$$
\sum_{k=0}^n(-1)^k{r\choose k}{r\choose n-k}=\begin{cases}0& n\text{ odd,}\\ \\ (-1)^{\frac{n}{2}}\begin{pmatrix} r \\ \frac{n}{2} \end{pmatrix}& n\text{ even,}\end{cases}
$$ 
implies the identity for the Kravchuk polynomials
\begin{gather*}
\mathcal{K}_n(j,2j)=\sum_{k=0}^n (-1)^k {j \choose k} {j \choose n-k}=\begin{cases} 0, \text{ $n$  odd}, \\ (-1)^{\frac{n}{2}} \begin{pmatrix} j \\ \frac{n}{2} \end{pmatrix}, \text{ $n$  even}. \end{cases}
\end{gather*}
Then
\begin{gather*}
\mathbi{e}_{2j}(0)=\sum_{k=0}^{2j} i^k \mathcal{K}_k\left(j,2j \right)a_{2j-k,k}=\sum_{k=0}^{j} i^{2k} \mathcal{K}_{2k}\left(j,2j \right)a_{2j-2k,2k}=\sum_{k=0}^{j} \binom{j}{k} a_{2j-2k, 2k}.
\end{gather*}
Note that this invariant for the first time was found in   \cite[page 358]{Ell}  and then was found again in  \cite[page 182]{Hu}.

$(ii)$ We have

\begin{gather*}
\mathbi{e}_{n}({si})=\sum_{j=0}^n i^j \mathcal{K}_j\left(\frac{1}{2}(n-s),n \right)a_{n-j,j}= \sum_{j=0}^{\left[\frac{n}{2}\right]} i^{2j} \mathcal{K}_{2j}\left(\frac{1}{2}(n-s),n \right)a_{n-2j,2j}+\\+\sum_{j=0}^{\left[\frac{n}{2}\right]} i^{2j+1} \mathcal{K}_{2j+1}\left(\frac{1}{2}(n-s),n \right)a_{n-2j-1,2j+1}=\\=\sum_{j=0}^{\left[\frac{n}{2}\right]} (-1)^j \mathcal{K}_{2j}\left(\frac{1}{2}(n-s),n \right)a_{n-2j,2j}+i \cdot \sum_{j=0}^{\left[\frac{n}{2}\right]} (-1)^j \mathcal{K}_{2j+1}\left(\frac{1}{2}(n-s),n \right)a_{n-2j-1,2j+1}.
\end{gather*}

Since $\mathbi{e}_{n}(-{si})=\overline{\mathbi{e}_{n}({si})}$, then 
\begin{gather*}
\mathbi{e}_{n}({si})\mathbi{e}_{n}(-{si})=Re(\mathbi{e}_{n}({si}))^2+Im(\mathbi{e}_{n}({si}))^2=\\=\left(\sum_{j=0}^{\left[\frac{n}{2}\right]} (-1)^j \mathcal{K}_{2j}\left(\frac{1}{2}(n-s),n \right)a_{n-2j,2j}\right)^2+\left( \sum_{j=0}^{\left[\frac{n}{2}\right]} (-1)^j \mathcal{K}_{2j+1}\left(\frac{1}{2}(n-s),n \right)a_{n-2j-1,2j+1} \right)^2.
\end{gather*}

\end{proof}

\end{document}